\documentclass[10pt]{article}
\usepackage[a4paper, margin=2.7cm]{geometry}

\usepackage[utf8]{inputenc}

\usepackage[T1]{fontenc}

\usepackage[english]{babel}

\usepackage{xcolor}
\definecolor{espblack}{RGB}{0,0,0}
\definecolor{espwhite}{RGB}{255,255,255}
\definecolor{espgray}{RGB}{206,206,206}
\definecolor{esplightgray}{RGB}{233,233,233}
\definecolor{espdarkgray}{RGB}{150,150,150}
\definecolor{espblue}{RGB}{11,93,174}
\definecolor{esplightblue}{RGB}{59,175,236}
\definecolor{espdarkblue}{RGB}{6,26,64}
\definecolor{espred}{RGB}{206,62,21}
\definecolor{esplightred}{RGB}{206,62,21}
\definecolor{espdarkred}{RGB}{61,19,8}
\definecolor{espyellow}{RGB}{232,163,26}
\definecolor{espgreen}{RGB}{100,161,27}
\definecolor{esplightgreen}{RGB}{149,198,35}
\definecolor{espdarkgreen}{RGB}{49,92,43}
\definecolor{esppurple}{RGB}{106,20,125}
\definecolor{esplightpurple}{RGB}{197,137,232}
\definecolor{espdarkpurple}{RGB}{50,14,59}
\usepackage{tikz}

\usepackage{tikzscale}

\usepackage{pgfplots}

\usepackage{pgfplotstable}

\usetikzlibrary{calc}

\usetikzlibrary{positioning}

\usetikzlibrary{backgrounds}

\usepgfplotslibrary{statistics}

\usepgfplotslibrary{fillbetween}

\pgfplotsset{
  compat=1.15,
  mps basic/.style={
    xlabel near ticks,
    xlabel style={font=\footnotesize},
    ylabel near ticks,
    ylabel style={font=\footnotesize},
    xmajorgrids,
    major x grid style={dotted},
    ymajorgrids,
    major y grid style={dotted},
    tick label style={font=\footnotesize}
  },
  mps scientific x/.style={
    x tick label style={
      /pgf/number format/sci
    }
  },
  mps scientific y/.style={
    y tick label style={
      /pgf/number format/sci
    }
  },
  mps fixed x/.style={
    x tick label style={
      /pgf/number format/.cd,
      fixed,
      fixed zerofill,
      precision=6,
      /tikz/.cd
    }
  },
  mps fixed y/.style={
    y tick label style={
      /pgf/number format/.cd,
      fixed,
      fixed zerofill,
      precision=6,
      /tikz/.cd
    }
  }
}
\usepackage{amsmath, amssymb}

\usepackage{bm}

\usepackage{mathtools}

\usepackage{etoolbox}

\usepackage{nicefrac}

\mathtoolsset{mathic}

\DeclarePairedDelimiterX\set[1]{\{}{\}}{%
  \ifblank{#1}{\:\ldots\:}{#1}
}

\DeclarePairedDelimiterX{\abs}[1]{\vert}{\vert}{
  \ifblank{#1}{\:\cdot\:}{#1}
}

\DeclarePairedDelimiterX{\norm}[1]{\lVert}{\rVert}{
  \ifblank{#1}{\:\cdot\:}{#1}
}
\DeclarePairedDelimiterXPP{\lnorm}[1]{}{\lVert}{\rVert}{_{2}}{
  \ifblank{#1}{\:\cdot\:}{#1}
}
\DeclarePairedDelimiterXPP{\infnorm}[1]{}{\lVert}{\rVert}{_{\infty}}{
  \ifblank{#1}{\:\cdot\:}{#1}
}
\DeclarePairedDelimiterXPP{\pnorm}[2]{}{\lVert}{\rVert}{
  \ifblank{#2}{_{p}}{_{#2}}}{
  \ifblank{#1}{\:\cdot\:}{#1}
}

\DeclarePairedDelimiterX{\inner}[2]{\langle}{\rangle}{
  \ifblank{#1}{\:\cdot\:}{#1},\ifblank{#2}{\:\cdot\:}{#2}
}

\DeclarePairedDelimiterXPP{\prob}[1]{\mathbb{P}}{[}{]}{}{
  
  \ifblank{#1}{\:\cdot\:}{#1}
}

\DeclarePairedDelimiterXPP{\expv}[1]{\mathbb{E}}{[}{]}{}{
  
  \ifblank{#1}{\:\cdot\:}{#1}
}
\usepackage{amsthm}

\newcounter{mpstheorem}
\newtheorem{theorem}[mpstheorem]{Theorem}
\newcounter{mpslemma}
\newtheorem{lemma}[mpslemma]{Lemma}
\newcounter{mpscorollary}
\newtheorem{corollary}[mpscorollary]{Corollary}

\usepackage{caption}

\usepackage{subcaption}

\usepackage{adjustbox}

\usepackage{rotating}
\input{preamble/mpsmiscellaneous}

\acrodef{ABIT*}{Advanced BIT*}
\acrodef{AIT*}{Adaptively Informed Trees}
\acrodef{BIT*}{Batch Informed Trees}

\usepackage[numbers,sort&compress]{natbib}

\pgfplotsset{compat=1.15}

\usepackage{xparse}
\usepackage{etoolbox}

\newcommand{\DeclareVariable}[2]{
  \NewDocumentCommand#1{ o o } {
    \IfNoValueTF{##1} {
      #2
    } {
      \ifblank{##1} {
        \IfNoValueTF{##2} {
          #2
        } {
          {#2}^{##2}
        }
      } {
        \IfNoValueTF{##2} {
          {#2}_{##1}
        } {
          {#2}_{##1}^{##2}
        }
      }
    }
  }
}

\newcommand{\DeclareFunction}[2]{
  \NewDocumentCommand#1{ s m } {
    \IfBooleanTF{##1} {
      #2
    }{
      \ifblank{##2}{
        #2\left(\;\cdot\;\right)
      }{
        #2\left(##2\right)
      }
    }
  }
}

\newcommand{\DeclareOperator}[2]{
  \NewDocumentCommand#1{ m } {
    \mathop{}\!\mathrm{#2}##1
  }
}

\DeclareVariable{\state}{\bm{\mathrm{x}}}
\DeclareVariable{\states}{X}
\DeclareVariable{\valstates}{\states[\mathrm{valid}]}
\DeclareVariable{\invstates}{\states[\mathrm{invalid}]}
\DeclareVariable{\paths}{\Sigma}
\DeclareVariable{\patharg}{t}
\DeclareVariable{\dist}{\delta}
\DeclareVariable{\arclength}{l}
\DeclareVariable{\lbarclength}{l_{\mathrm{lb}}}

\DeclareFunction{\cost}{c}
\DeclareFunction{\clearance}{\delta}
\DeclareFunction{\closure}{\mathrm{closure}}
\DeclareFunction{\path}{\sigma}
\DeclareFunction{\pathalias}{\sigma}

\DeclareOperator{\diff}{d}

\usepackage{fp}

\usetikzlibrary{external}
\usetikzlibrary{decorations.markings}
\usetikzlibrary{decorations.pathreplacing}
\usetikzlibrary{fixedpointarithmetic}
\usetikzlibrary{patterns}
\tikzset{
  state/.style={circle, inner sep = 1.0pt, fill = black},
  obstacle/.style={fill = espdarkgray, draw = espdarkgray},
  path/.style={black},
  bounds/.style={dashed}
}

\usepackage{datetime2}
\usepackage{eso-pic}
\AddToShipoutPictureBG{%
  \AtTextUpperLeft{%
    \raisebox{1.5em}{%
      \parbox[b][2\baselineskip][t]{\textwidth}{%
        \raggedleft\normalfont%
        TR-2021-MPS001 \\
        Revision 1.1\\
        \today%
      }%
    }%
  }%
}

\usepackage[hidelinks]{hyperref}

\title{\vspace*{-0.5em}\LARGE\bf Admissible heuristics\\
  for obstacle clearance optimization objectives}
\date{}

\author{Marlin P.\ Strub and Jonathan D.\ Gammell\\
  Estimation, Search, and Planning (ESP) Group\\
  Oxford Robotics Institute (ORI)\\
  University of Oxford\\
  \texttt{\small(mstrub|gammell)@robots.ox.ac.uk}}

\begin{document}

\maketitle

\begin{abstract}
  Obstacle clearance in state space is an important optimization objective in
  path planning because it can result in safe paths. This technical report
  presents admissible solution- and path-cost heuristics for this objective,
  which can be used to improve the performance of informed path planning
  algorithms.
\end{abstract}

\section{Preliminaries}%
\label{sec:preliminaries}

Let \( X \) be a state space, \( \invstates \subseteq \states \) be the subset
of invalid states, and \( \valstates \coloneqq \states \setminus \invstates \)
be the subset of valid states. Let \( \path*{} \colon [0, l] \to \valstates \)
be a path, i.e., a continuous function parameterized by its arc length,
\( \arclength < \infty \), and let \( \paths \) denote the set of all valid paths. Let
\( \clearance*{} \colon \valstates \to (0, \infty) \) be the clearance of a
valid state and let \( \cost*{} \colon \paths \to [0, \infty) \) be the
reciprocal clearance cost of a path,
\begin{align*}
  \clearance{\state} &\coloneqq \min_{\state[][\prime] \in X_{\text{invalid}}} \norm{\state - \state[][\prime]} & \cost{\path*{}} &\coloneqq \int_{0}^{\arclength}
  \frac{1}{\clearance{\path{\patharg}}} \diff{\patharg}.\stepcounter{equation}\tag{\theequation}\label{eq:reciprocal_obstacle_clearance_cost}
\end{align*}
It is assumed that no state on a path of interest has clearance of exactly
zero.

\begin{lemma}[An upper bound on state clearance]\label{thm:state-clearance}
  Let \( \path*{} \in \paths \) be a path with arc length \( \arclength \). Let
  \( \path{\patharg[1]} \in \states \) be the state
  \( 0 \leq \patharg[1] \leq \arclength \) along this path, and let \( \dist[1] \) be
  the known clearance of this state,
  \( \dist[1] \coloneqq \clearance{\path{\patharg[1]}} \). The clearance of any
  state on the same path, \( \clearance{\path{\patharg}} \), is then upper
  bounded by
\begin{align*}
  \clearance{\path{\patharg}} \leq \dist[1] + \abs{\patharg[1] - \patharg}.\stepcounter{equation}\tag{\theequation}\label{eq:upper_bound_clearance}
\end{align*}
\end{lemma}

\begin{proof}
  (Figure~\ref{fig:lemma-1}) Let \( \state[1][\prime] \in X_{\text{invalid}} \)
  be one of the closest invalid states of the state
  \( \state[1] \coloneqq \path{\patharg[1]} \),
  \begin{align*}
    \state[1][\prime] \coloneqq \mathop{\arg\min}_{\state[][\prime] \in X_{\text{invalid}}}\norm{\state[1] - \state[][\prime]} \,\implies\, \dist[1] = \norm{\state[1] - \state[1][\prime]}.
  \end{align*}
  Because \( \path{\patharg} \) is parametrized by arc length, any state on the
  path is at most \( \abs{\patharg[1] - \patharg} \) away from the state
  \( \state[1] \) and therefore at most
  \( \norm{\state[1] - \state[1][\prime]} + \abs{\patharg[1] - \patharg} \)
  away from the state \( \state[1][\prime] \). Since \( \state[1][\prime] \) is
  an invalid state, the distance \( \dist[1] + \abs{\patharg[1] - \patharg} \)
  provides an upper bound on the true clearance of any state on the path.
\end{proof}

\begin{figure}[t]
  \centering
  \begin{tikzpicture}[ scale = 4, > = latex,
    arc length/.style={densely dotted, transform canvas={yshift=#1}},
    dist arrow/.style={->, densely dotted},
    dist label/.style={below, inner sep = 2pt, sloped}]

    \draw [obstacle] plot [smooth cycle, tension = 1.2]
                     coordinates {(-0.32, 0.17) (0.18, -0.13) (0.08, 0.57)};
    \node (xinv) [inner sep = 0.0pt] at (-0.15, -0.08) {};
    \node (xinvtext) [below left = 8pt and -5pt of xinv, inner sep = 2pt] {\( \invstates \)};
    \draw (xinvtext) -- (xinv);

    \node (x1) [state] at (0.5, 0.7) {};
    \node (x2) [state] at (2.5, -0.1) {};
    \node (x1text) [below left = 8pt and -5pt of x1, inner sep = 2pt] {\( \path{0}\;\, \)};
    \node (x2text) [above left = 8pt and -5pt of x2, inner sep = 2pt] {\( \path{l}\;\, \)};
    \draw (x1text) -- (x1);
    \draw (x2text) -- (x2);

    \draw [path] plot [smooth, tension = 1.1]
    coordinates {
      (x1)
      ($(x1) + (0.7, 0.1)$)
      ($(x2) - (0.8, -0.2)$)
      (x2)
    };

    \begin{scope}
      \clip (0.524, 1.0) rectangle (1.405, 0.0);
      \draw [fixed point arithmetic,
             postaction = {decoration = {markings,
                 mark = at position 0.439 with {\arrow{>}}}, decorate},
             arc length={-3pt}] plot [smooth, tension = 1.1, thick]
                                coordinates {(x1) ($(x1) + (0.695, 0.105)$) ($(x2) - (0.83, -0.2)$) (x2)};
    \end{scope}
    \node at (0.75, 0.72) {\( \patharg[1] \)};
    \begin{scope}
      \clip (0.0, 1.0) rectangle (1.615, 0.0);
      \draw [fixed point arithmetic,
             postaction = {decoration = {markings,
                 mark = at position 0.615 with {\arrow{>}}}, decorate},
             arc length={ 3pt}] plot [smooth, tension = 1.12, thick]
                                coordinates {(x1) ($(x1) + (0.715, 0.090)$) ($(x2) - (0.778, -0.180)$) (x2)};
    \end{scope}
    \node at (0.75, 0.94) {\( \patharg \)};

    \node (t)  [state] at (1.611, 0.171)  {};
    \node (t1) [state] at (1.435, 0.5) {};
    \node (tprime) [state] at (1.850, 0.575) {};
    \node (ttext)  [below left = 8pt and -5pt of t,  inner sep = 2pt] {\( \path{\patharg}\;\, \)};
    \node (t1text) [below left = 8pt and -5pt of t1, inner sep = 2pt] {\( \path{\patharg[1]} \)};
    \draw (ttext)  -- (t);
    \draw (t1text) -- (t1);
    \draw [dist arrow] (t1) -- node [dist label] {\( \dist[1] \)} (0.24, 0.28);
    \draw [dist arrow] (t1) -- node [dist label] {\( \abs{\patharg[1] - \patharg} \)} (tprime);
    \draw [dist arrow] (t) -- node [dist label] {\( \clearance{\path{\patharg}} \)} (0.255, 0.09);

    \node (path) [inner sep = 0pt] at (2.0, -0.045) {};
    \node (pathtext) [below left = 8pt and 2pt of path, inner sep = 2pt] {\( \path*{} \)};
    \draw (path) -- (pathtext);

    \begin{scope}[on background layer]
      \filldraw [espgray] (t1) circle (0.422);
    \end{scope}
  \end{tikzpicture}
  \caption{An illustration of Lemma~\ref{thm:state-clearance}. Any state
    \( \path{\patharg} \) on the path \( \path*{} \) can not be farther from
    the state \( \path{\patharg[1]} \) than \( \abs{\patharg[1] - \patharg} \)
    and must be within the light gray shaded area. The clearance
    \( \clearance{\path{\patharg}} \) of any state on the path can therefore
    not be larger than the clearance
    \( \dist[1] \coloneqq \clearance{\path{\patharg[1]}} \) of the state
    \( \path{\patharg[1]} \) plus the distance from that state, i.e.,
    \( \clearance{\path{\patharg}} \leq \dist[1] + \abs{\patharg[1] -
      \patharg}. \)}%
  \label{fig:lemma-1}
\end{figure}


\section{Solution-cost heuristics}%
\label{sec:solution-cost-heuristics}

This section presents two lower bounds on the cost of an optimal path between
two states. These bounds can be used as admissible cost-to-go heuristics in
informed planners, such as \accite{BIT*}{gammell_ijrr2020},
\accite{ABIT*}{strub_icra2020a}, and \accite{AIT*}{strub_icra2020b}, if a lower
bound on the arc length of the optimal path between two states is known, e.g.,
the Euclidean distance.

\begin{lemma}[An admissible solution-cost heuristic when the clearance of one
  end state is known]\label{thm:end-state}
  Let \( \path*{} \in \paths \) be a path whose arc length, \( \arclength \),
  is lower bounded by \( \lbarclength \leq \arclength \). Let the clearance of
  the start or goal state of this path be known and denoted by
  \( \dist[1] \coloneqq \clearance{\path{0 \text{ or } \arclength}} \). The
  reciprocal clearance cost \( \cost{\path*{}} \) of the path \( \path*{} \) is
  then lower bounded by
  \begin{align*}
    \cost{\path*{}} \geq \ln\left(\frac{\dist[1] + \lbarclength}{\dist[1]} \right).
  \end{align*}
\end{lemma}

\begin{proof}
  The bounds are computed by setting \( \patharg[1] = 0 \) or
  \( \patharg[1] = \arclength \) in
  Lemma~\ref{thm:state-clearance} and replacing the clearance function in the
  integrand of the reciprocal clearance
  cost~\eqref{eq:reciprocal_obstacle_clearance_cost} with the upper bound on
  the clearance~\eqref{eq:upper_bound_clearance}. First let
  \( \patharg[1] = 0 \) (Figure~\ref{fig:start-state}). Then by
  Lemma~\ref{thm:state-clearance}
  \begin{align*}
    \clearance{\path{\patharg}} \leq \dist[1] + t,
  \end{align*}
  and the lower bound on the cost is
  \begin{align*}
    \cost{\path*{}} = \int_{0}^{\arclength} \frac{1}{\clearance{\path{\patharg}}}\diff{\patharg}
              &\geq \int_{0}^{\arclength} \frac{1}{\dist[1] + \patharg} \diff{\patharg} \\
              &= {\left[ \ln\left( \dist[1] + \patharg \right) \right]}_{0}^{\arclength} \\
              &= \ln\left( \dist[1] + \arclength \right) - \ln\left( \dist[1] \right) \\
              &= \ln\left( \frac{\dist[1] + \arclength}{\dist[1]} \right) \\
              &\geq \ln\left( \frac{\dist[1] + \lbarclength}{\dist[1]} \right).
  \end{align*}
  Similarly, let \( \patharg[1] = \arclength \) (Figure~\ref{fig:end-state}). Then by
  Lemma~\ref{thm:state-clearance}
  \begin{align*}
    \clearance{\path{\patharg}} \leq \dist[1] + \arclength - \patharg,
  \end{align*}
  and the lower bound on the cost is again
  \begin{align*}
    \cost{\path*{}} = \int_{0}^{\arclength} \frac{1}{\clearance{\path{\patharg}}}\diff{\patharg}
              &\geq \int_{0}^{\arclength} \frac{1}{\dist[1] + \arclength - \patharg} \diff{\patharg} \\
              &= {\left[ - \ln\left( \dist[1] + \arclength - \patharg \right) \right]}_{0}^{\arclength} \\
              &= -\ln\left( \dist[1] \right) - \left( - \ln\left( \dist[1]
                + \arclength \right) \right) \\
              &= \ln\left( \frac{\dist[1] + \arclength}{\dist[1]} \right) \\
              &\geq \ln\left( \frac{\dist[1] + \lbarclength}{\dist[1]} \right).
  \end{align*}
\end{proof}

\begin{theorem}[An admissible solution-cost heuristic when the clearance of both
  end states is known]\label{thm:end-states}
  Let \( \path*{} \) be a path whose arc length, \( \arclength \), is lower
  bounded by \( \lbarclength \leq \arclength \). Let
  \( \path{\patharg[1] = 0} \in X \) and
  \( \path{\patharg[2] = \arclength} \in X \) be its start and goal states with
  clearances \( \dist[1] \) and \( \dist[2] \), respectively. The reciprocal
  clearance cost \( \cost{\path*{}} \) of the path \( \path*{} \) is then lower
  bounded by
  \begin{align*}
    \cost{\path*{}} \geq \ln\left( \frac{{\left( \dist[1] + \dist[2] + \lbarclength \right)}^{2}}{4\dist[1]\dist[2]} \right).
  \end{align*}
\end{theorem}

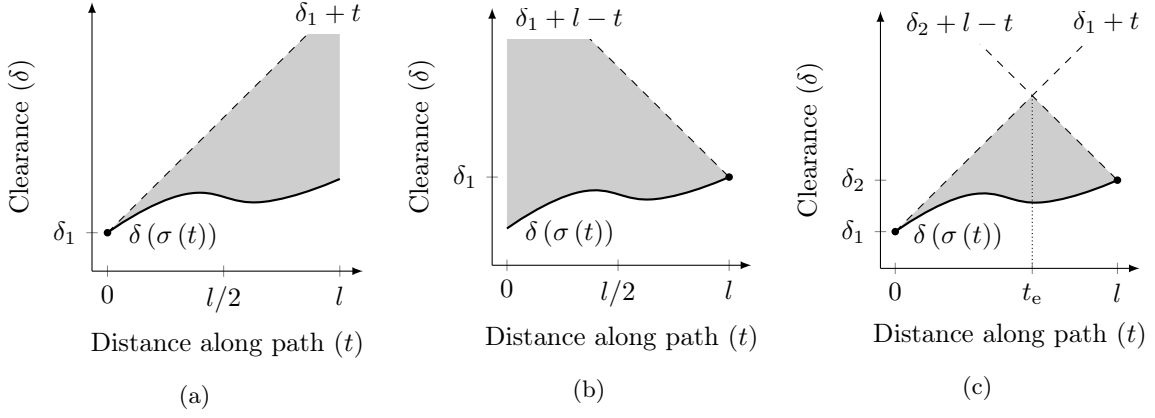
\begin{figure}[t]
  \centering
  \begin{subfigure}{0.33\textwidth}
    \begin{tikzpicture}
  \begin{axis} [
    width = \textwidth,
    height = \textwidth,
    ymin = 1.0,
    ymax = 4.5,
    xmin = -0.2,
    xmax = 3.3,
    xtick = {0, 1.5, 3},
    xticklabels = {\( 0 \), \( l/2 \), \( l \)},
    ytick = {1.5},
    yticklabels = {\( \dist[1] \)},
    axis x line = bottom,
    axis y line = left,
    axis line style = {-latex},
    xlabel = {Distance along path (\( t \))},
    ylabel = {Clearance (\( \clearance*{} \))},
    axis on top
    ]
    \node (state1) [state] at (0.0, 1.5) {};

    \addplot [
      smooth,
      tension = 0.8,
      thick,
      name path = clearance
    ] coordinates {
      (0.0, 1.5)
      (1.0, 2.0)
      (2.0, 1.9)
      (3.0, 2.2)
    } node [below = 5pt, pos = 0.3] {\( \clearance{\pathalias{\patharg}} \)};

    \addplot [
      samples = 2,
      domain = 0:3,
      name path = bound1,
      bounds
    ] {
      x + 1.5
    } node [pos = 0.95, fill = white] {\( \dist[1] + \patharg \)};

    \addplot [
      fill = espgray
    ] fill between [of = bound1 and clearance];
  \end{axis}
\end{tikzpicture}

  \caption{}%
  \label{fig:start-state}
\end{subfigure}
\begin{subfigure}{0.32\textwidth}
  \begin{tikzpicture}
  \begin{axis} [
    width = \textwidth,
    height = \textwidth,
    ymin = 1.0,
    ymax = 4.5,
    xmin = -0.2,
    xmax = 3.3,
    xtick = {0, 1.5, 3},
    xticklabels = {\( 0 \), \( l/2 \), \( l \)},
    ytick = {2.2},
    yticklabels = {\( \dist[1] \)},
    axis x line = bottom,
    axis y line = left,
    axis line style = {-latex},
    xlabel = {Distance along path (\( t \))},
    ylabel = {Clearance (\( \clearance*{} \))},
    axis on top
    ]
    \node (state2) [state] at (3.0, 2.2) {};

    \addplot [
      smooth,
      tension = 0.8,
      thick,
      name path = clearance
    ] coordinates {
      (0.0, 1.5)
      (1.0, 2.0)
      (2.0, 1.9)
      (3.0, 2.2)
    } node [below = 5pt, pos = 0.3] {\( \clearance{\pathalias{\patharg}} \)};

    \addplot [
      samples = 2,
      domain = 0:3,
      name path = bound2,
      bounds
    ] {
      -(x - 3) + 2.2
    } node [pos = 0.28, fill = white] {\( \dist[1] + l - \patharg \)};

    \addplot [
      fill = espgray
    ] fill between [of = bound2 and clearance];
  \end{axis}
\end{tikzpicture}

  \caption{}%
  \label{fig:end-state}
\end{subfigure}
\begin{subfigure}{0.32\textwidth}
  \begin{tikzpicture}
  \begin{axis} [
    width = \textwidth,
    height = \textwidth,
    ymin = 1.0,
    ymax = 4.5,
    xmin = -0.2,
    xmax = 3.3,
    xtick = {0, 1.85, 3},
    xticklabels = {\( 0 \), \( \patharg[\text{e}] \), \( l \)},
    ytick = {1.5, 2.2},
    yticklabels = {\( \dist[1] \), \( \dist[2] \)},
    axis x line = bottom,
    axis y line = left,
    axis line style = {-latex},
    xlabel = {Distance along path (\( t \))},
    ylabel = {Clearance (\( \clearance*{} \))},
    axis on top
    ]
    \node (state1) [state] at (0.0, 1.5) {};
    \node (state2) [state] at (3.0, 2.2) {};
    \coordinate (statee) at (1.85, 3.35) {};
    \draw [densely dotted] (statee) -- (1.85, 0.0);

    \addplot [
      smooth,
      tension = 0.8,
      thick,
      name path = clearance
    ] coordinates {
      (0.0, 1.5)
      (1.0, 2.0)
      (2.0, 1.9)
      (3.0, 2.2)
    } node [below = 5pt, pos = 0.3] {\( \clearance{\pathalias{\patharg}} \)};

    \addplot [
      samples = 2,
      domain = 0:3,
      name path = bound1,
      bounds
    ] {
      x + 1.5
    } node [pos = 0.95, fill = white] {\( \dist[1] + \patharg \)};

    \addplot [
      samples = 2,
      domain = 0:3,
      name path = bound2,
      bounds
    ] {
      -(x - 3) + 2.2
    } node [pos = 0.29, fill = white] {\( \dist[2] + l - \patharg \)};

    \addplot [
      fill = espgray
      ] fill between [of = bound2 and clearance, soft clip = {domain = 1.85:3}];

    \addplot [
      fill = espgray
    ] fill between [of = bound1 and clearance, soft clip = {domain = 0.01:1.85}];

  \end{axis}
\end{tikzpicture}

    \caption{}%
  \label{fig:end-states}
\end{subfigure}
\caption{Illustrations of the solution-cost heuristics. The gray area indicates
  the error of the heuristic.}%
\label{fig:solution-cost-heuristics}
\end{figure}


\begin{proof}
  (Figure~\ref{fig:end-states}) The clearance of any state
  \( \path{\patharg}, 0 \leq \patharg \leq l, \) on the path is upper bounded
  by both clearances \( \dist[1] \) and \( \dist[2] \) according to
  Lemma~\ref{thm:state-clearance},
  \begin{align*}
    \clearance{\path{\patharg}} &\leq \min\{\dist[1] + \abs{\patharg[1] -
    \patharg}, \dist[2] + \abs{\patharg[2] - \patharg} \}\\ &= \min\{ \dist[1] + \patharg,  \dist[2] + \arclength - \patharg \}
  \end{align*}
  A lower bound on the reciprocal clearance cost \( \cost{\path*{}} \) of the
  path \( \path*{} \) can therefore be computed by
  \begin{align*}
    \cost{\path*{}} &= \int_{0}^{\arclength} \frac{1}{\clearance{\path{\patharg}}} \diff{\patharg} \geq \int_{0}^{\arclength} \frac{1}{\min\left\{ \dist[1] +
                      \patharg, \dist[2] + \arclength - \patharg \right\}} \diff{\patharg}.\stepcounter{equation}\tag{\theequation}\label{eq:integral_with_min}
  \end{align*}
  Since \( \dist[1] + \patharg \) is strictly monotonically increasing and
  \( \dist[2] + \arclength - \patharg \) is strictly monotonically decreasing, the two
  bounds must intersect at some point, \( \patharg[\text{e}] \),
  \begin{align*}
    \dist[1] + \patharg[\text{e}] = \dist[2] + \arclength - \patharg[\text{e}] \quad\implies\quad
    \patharg[\text{e}] = \frac{\dist[2] - \dist[1] + \arclength}{2}.
  \end{align*}
  This intersection must lie within the integration limits because by
  Lemma~\ref{thm:state-clearance} we have
  \begin{align*}
    \clearance{\path{\patharg}} \leq \dist[2] + \arclength - t \,\implies\, \clearance{\path{0}} \leq \dist[2] + \arclength \,\implies\, \dist[1] \leq \dist[2] + \arclength \,\implies\, \patharg[\text{e}] \geq 0\phantom{.}\\
    \clearance{\path{\patharg}} \leq \dist[1] + t \,\implies\, \clearance{\path{\arclength}} \leq \dist[1] + \arclength \,\implies\, \dist[2] \leq \dist[1] + \arclength \,\implies\, \dist[2] - \dist[1] \leq \arclength \,\implies\, \patharg[\text{e}] \leq l.
  \end{align*}
  The minimum in~\eqref{eq:integral_with_min} can therefore be written as
  \begin{align*}
    \min\left\{\dist[1] + \patharg, \dist[2] + \arclength - \patharg \right\} =
    \begin{cases}
      \dist[1] + \patharg     & \text{if } \patharg \leq \patharg[\text{e}] \\
      \dist[2] + \arclength - \patharg & \text{otherwise},
    \end{cases}
  \end{align*}
  and the definite integral~\eqref{eq:integral_with_min} can be evaluated to
  \begin{align*}
    \cost{\path*{}} = \int_{0}^{\arclength} \frac{1}{\clearance{\path{\patharg}}}\diff{\patharg} &\geq \int_{0}^{\arclength} \frac{1}{\min\left\{ \dist[1] + \patharg, \dist[2] + \arclength - \patharg \right\}}\\
              &= \int_{0}^{\patharg[\text{e}]} \frac{1}{\dist[1] + \patharg} \diff{\patharg}
                + \int_{\patharg[\text{e}]}^{\arclength} \frac{1}{\dist[2] + \arclength - \patharg} \diff{\patharg} \\
              &= {\left[ \ln\left( \dist[1] + \patharg \right) \right]}_{0}^{\patharg[\text{e}]}
                + {\left[ -\ln\left( \dist[2] + \arclength - \patharg \right) \right]}_{\patharg[\text{e}]}^{\arclength} \\
              &= \ln\left( \dist[1] + \frac{\dist[2] - \dist[1] + \arclength}{2} \right) - \ln\left( \dist[1] \right) + \left( -\ln\left( \dist[2] \right) - \left(- \ln\left( \dist[2] + \arclength - \frac{\dist[2] - \dist[1] + \arclength}{2} \right) \right) \right)\\
              &= \ln\left( \frac{\dist[1] + \dist[2] + \arclength}{2\dist[1]} \right) + \ln\left( \frac{\dist[1] + \dist[2] + \arclength}{2\dist[2]} \right)\\
              &= \ln\left( \frac{{\left(\dist[1] + \dist[2] + \arclength\right)}^{2}}{4\dist[1]\dist[2]} \right) \\
              &\geq \ln\left( \frac{{\left(\dist[1] + \dist[2] + \lbarclength\right)}^{2}}{4\dist[1]\dist[2]} \right).
  \end{align*}
\end{proof}

\section{Path-cost heuristics}%
\label{sec:path-cost-heuristics}

Informed sampling-based planning algorithms, such as \ac{BIT*}, \ac{ABIT*}, and
\ac{AIT*}, also use path-cost heuristics, i.e., estimates of the unknown cost
of known paths, for example when ordering their search queues. The
solution-cost heuristics of Section~\ref{sec:solution-cost-heuristics} can be
made more accurate for known paths by sampling additional states along the path
and computing their clearance. This improves performance if the evaluation of
the true edge cost is computationally expensive.

\begin{lemma}[An admissible path-cost heuristic when the clearance of any state
  on the path is known]\label{thm:single-known-state}
  Let \( \path*{} \in \paths \) be a path with arc length \( \arclength \). Let
  \( \path{\patharg[1]} \in \states \) be the state
  \( 0 \leq \patharg[1] \leq \arclength \) along this path, and let
  \( \dist[1] \coloneqq \clearance{\path{\patharg[1]}} \) be the known
  clearance of this state. The reciprocal clearance cost \( \cost{\path*{}} \)
  of the path \( \path*{} \) is then lower bounded by
  \begin{align*}
     \cost{\path*{}} \geq \ln\left( \frac{\dist[1] + \patharg[1]}{\dist[1]} \right) + \ln\left(\frac{\dist[1] + \arclength - \patharg[1]}{\dist[1]} \right).\stepcounter{equation}\tag{\theequation}\label{eq:lower-bound-single-state}
  \end{align*}
\end{lemma}

\begin{proof}
  (Figure~\ref{fig:single-state}) The lower bound is computed by replacing the
  clearance function in the integrand of the reciprocal clearance
  cost~\eqref{eq:reciprocal_obstacle_clearance_cost} with the upper bound on
  the clearance~\eqref{eq:upper_bound_clearance},
  \begin{align*}
    \cost{\path*{}} = \int_{0}^{\arclength} \frac{1}{\clearance{\path{\patharg}}}\diff{\patharg}
              &\geq \int_{0}^{\arclength} \frac{1}{\dist[1] + \abs{\patharg[1] - \patharg}} \diff{\patharg} \\
              &= \int_{0}^{\patharg[1]} \frac{1}{\dist[1] + \patharg[1] - \patharg} \diff{\patharg}
                  + \int_{\patharg[1]}^{\arclength} \frac{1}{\dist[1] + \patharg - \patharg[1]} \diff{\patharg}. \\
              &= {\left[ -\ln\left( \dist[1] + \patharg[1] - \patharg \right) \right]}_{0}^{\patharg[1]}
                  + {\left[ \ln\left( \dist[1] + \patharg - \patharg[1] \right) \right]}_{\patharg[1]}^{\arclength} \\
              &= -\ln\left( \dist[1] \right) - \left(-\ln\left( \dist[1] + \patharg[1] \right) \right)
                 +\ln\left( \dist[1] + \arclength - \patharg[1] \right) - \ln\left( \dist[1] \right) \\
              &= \ln\left( \frac{\dist[1] + \patharg[1]}{\dist[1]} \right) + \ln\left( \frac{\dist[1] + \arclength - \patharg[1]}{\dist[1]} \right).
  \end{align*}
  Note that this reduces to the result of Lemma~\ref{thm:state-clearance} when
  \( \patharg[1] = 0 \) or \( \patharg[1] = \arclength \).
\end{proof}

\begin{theorem}[An admissible path-cost heuristic when the clearance of multiple
  states on the path is known]\label{thm:arbitrary-num-states}
  Let \( \path*{} \in \paths \) be a path with arc length \( \arclength \).
  Let
  \( 0 \leq \patharg[1] < \patharg[2] < \cdots < \patharg[n] \leq \arclength
  \) be a sequence of \( n \) numbers between \( 0 \) and \( \arclength \)
  whose associated states on the path have known clearance,
  \( \dist[i] \coloneqq \clearance{\path{\patharg[i]}} \) for
  \( i = 1, 2, 3, \hdots, n \). The reciprocal clearance cost
  \( \cost{\path*{}} \) of the path \( \path*{} \) is then lower bounded by
  \begin{align*}
     \cost{\path*{}} \geq \ln\left( \frac{\dist[1] + \patharg[1]}{\dist[1]} \right) + \sum_{i = 1}^{n - 1} \ln\left( \frac{{\left( \dist[i] + \dist[i + 1] + \patharg[i + 1] - \patharg[i] \right)}^{2}}{4\dist[i]\dist[i + 1]} \right) +\ln\left( \frac{\dist[n] + \arclength - \patharg[n]}{\dist[n]} \right).
  \end{align*}
\end{theorem}

\begin{proof}
  (Figure~\ref{fig:many-states}) The proof follows from
  Lemma~\ref{thm:end-state} and Theorem~\ref{thm:end-states} by splitting the
  clearance cost into \( n + 1 \) segments between the states with known
  clearance,
  \begin{align*}
    \cost{\path*{}} &= \int_{0}^{\arclength} \frac{1}{\clearance{\path{\patharg}}}\diff{\patharg} \\
              &= \int_{0}^{\patharg[1]} \frac{1}{\clearance{\path{\patharg}}}\diff{\patharg} + \int_{\patharg[1]}^{\patharg[2]} \frac{1}{\clearance{\path{\patharg}}}\diff{\patharg} + \cdots + \int_{\patharg[n-1]}^{\patharg[n]} \frac{1}{\clearance{\path{\patharg}}}\diff{\patharg} + \int_{\patharg[n]}^{\arclength} \frac{1}{\clearance{\path{\patharg}}}\diff{\patharg}.
  \end{align*}

  Let the arc lengths of these segments be denoted by
  \begin{align*}
    l_{0} = \patharg[1], l_{1} = \patharg[2] - \patharg[1], \ldots, l_{n - 1} = \patharg[n] - t_{n - 1},
    l_{n} = \arclength - \patharg[n].
  \end{align*}
  The first segment is a path of arc length \( l_{0} \) with known clearance of
  its end state and therefore lower bounded by Lemma~\ref{thm:end-state},
  \begin{align*}
    \int_{0}^{\patharg[1]}
    \frac{1}{\clearance{\path{\patharg}}}\diff{\patharg}
      &\geq \ln\left( \frac{\dist[1] + \arclength_{0}}{\dist[1]} \right) = \ln\left( \frac{\dist[1] + \patharg[1]}{\dist[1]} \right).\stepcounter{equation}\tag{\theequation}\label{eq:first-segment}
  \end{align*}

  Similarly, the last segment is a path of arc length \( l_{n} \) with known
  clearance of its start state and therefore also lower bounded by
  Lemma~\ref{thm:end-state},
  \begin{align*}
    \int_{\patharg[n]}^{\arclength}
    \frac{1}{\clearance{\path{\patharg}}}\diff{\patharg}
      &\geq \ln\left( \frac{\dist[n] + \arclength_{n}}{\dist[n]} \right) \\
      &= \ln\left( \frac{\dist[n] + \arclength - \patharg[n]}{\dist[n]} \right).\stepcounter{equation}\tag{\theequation}\label{eq:last-segment}
  \end{align*}

  Each of the segments between \( \patharg[1] \) and \( \patharg[n] \) can be
  viewed as a path with known clearance at the end-states and is therefore
  lower bounded by the result of Theorem~\ref{thm:end-states}. Specifically,
  the segment from \( \patharg[i] \) to \( \patharg[i + 1] \) with
  \( i \in [1, n - 1] \) is lower bounded by
  \begin{align*}
    \int_{\patharg[i]}^{\patharg[i + 1]} \frac{1}{\clearance{\path{\patharg}}} \diff{\patharg}
      &\geq \ln\left( \frac{{\left( \dist[i] + \dist[i + 1] + \arclength[i] \right)}^{2}}{4\dist[i]\dist[i + 1]} \right)\\
      &= \ln\left( \frac{{\left( \dist[i] + \dist[i + 1] + \patharg[i + 1] - \patharg[i] \right)}^{2}}{4\dist[i]\dist[i + 1]} \right).\stepcounter{equation}\tag{\theequation}\label{eq:mid-segment}
  \end{align*}

  A lower bound on the path cost can be computed by adding the lower bounds~\eqref{eq:first-segment},~\eqref{eq:last-segment}
  and~\eqref{eq:mid-segment}
  \begin{align*}
    c(\path*{}) &= \int_{0}^{\arclength} \frac{1}{\clearance{\path{\patharg}}}\diff{\patharg} \\
                &= \int_{0}^{\patharg[1]} \frac{1}{\clearance{\path{\patharg}}}\diff{\patharg}
                  + \sum_{i = 1}^{n - 1} \int_{\patharg[1]}^{\patharg[i + 1]} \frac{1}{\clearance{\path{\patharg}}}\diff{\patharg}
                  + \int_{\patharg[n]}^{\arclength} \frac{1}{\clearance{\path{\patharg}}}\diff{\patharg} \\
                &\geq \ln\left( \frac{\dist[1] + \patharg[1]}{\dist[1]} \right)
                  + \sum_{i = 1}^{n - 1} \ln\left( \frac{{\left( \dist[i] + \dist[i + 1] + \patharg[i + 1] - \patharg[i] \right)}^{2}}{4\dist[i]\dist[i + 1]} \right)
                  + \ln\left( \frac{\dist[n] + \arclength - \patharg[n]}{\dist[n]} \right).
  \end{align*}
\end{proof}

\begin{figure}[t]
  \centering
  \begin{subfigure}{0.33\textwidth}
    \begin{tikzpicture}
  \begin{axis} [
    width = \textwidth,
    height = \textwidth,
    ymin = 1.0,
    ymax = 4.5,
    xmin = -0.2,
    xmax = 3.3,
    xtick = {0, 2, 3},
    xticklabels = {\( 0 \), \( \patharg[1] \), \( l \)},
    ytick = {1.9},
    yticklabels = {\( \dist[1] \)},
    axis x line = bottom,
    axis y line = left,
    axis line style = {-latex},
    xlabel = {Distance along path (\( t \))},
    ylabel = {Clearance (\( \clearance*{} \))},
    axis on top
    ]
    \node (state1) [state] at (2.0, 1.9) {};
    \draw [densely dotted] (state1) -- (2.0, 0.0);

    \addplot [
      smooth,
      tension = 0.8,
      thick,
      name path = clearance
    ] coordinates {
      (0.0, 1.5)
      (1.0, 2.0)
      (2.0, 1.9)
      (3.0, 2.2)
    } node [below = 5pt, pos = 0.3] {\( \clearance{\pathalias{\patharg}} \)};

    \addplot [
      samples = 100,
      domain = 0:3,
      name path = bound1,
      bounds
    ] {
      abs(x - 2.0) + 1.9
    } node [pos = 0.0, fill = white, anchor = south west] {\( \dist[1] + \abs{\patharg[1] - \patharg} \)};

    \addplot [
      fill = espgray
    ] fill between [of = bound1 and clearance];
  \end{axis}
\end{tikzpicture}

  \caption{}%
  \label{fig:single-state}
\end{subfigure}
\begin{subfigure}{0.32\textwidth}
  \begin{tikzpicture}
  \begin{axis} [
    width = \textwidth,
    height = \textwidth,
    ymin = 1.0,
    ymax = 4.5,
    xmin = -0.2,
    xmax = 3.3,
    xtick = {0, 1.5, 3},
    xticklabels = {\( 0 \), \( \patharg[i] \), \( l \)},
    ytick = {1.96},
    yticklabels = {\( \dist[i] \)},
    axis x line = bottom,
    axis y line = left,
    axis line style = {-latex},
    xlabel = {Distance along path (\( t \))},
    ylabel = {Clearance (\( \clearance*{} \))},
    axis on top
    ]
    \node (state1) [state] at (0.5, 1.8) {};
    \node (state2) [state] at (1.5, 1.96) {};
    \node (state3) [state] at (2.5, 2.0) {};
    \draw [densely dotted] (state2) -- (1.5, 0.0);

    \addplot [
      smooth,
      tension = 0.8,
      thick,
      name path = clearance
    ] coordinates {
      (0.0, 1.5)
      (1.0, 2.0)
      (2.0, 1.9)
      (3.0, 2.2)
    } node [below = 5pt, pos = 0.3] {\( \clearance{\pathalias{\patharg}} \)};

    \addplot [
      samples = 100,
      domain = 0:3,
      name path = bound1,
      bounds
    ] {
      abs(x - 0.5) + 1.80
    };
    \addplot [
      samples = 100,
      domain = 0:3,
      name path = bound2,
      bounds
    ] {
      abs(x - 1.5) + 1.96
    };
    \addplot [
      samples = 100,
      domain = 0:3,
      name path = bound3,
      bounds
    ] {
      abs(x - 2.5) + 2.00
    };

    \addplot [
      fill = espgray
    ] fill between [of = bound1 and clearance, soft clip = {domain = 0.01:1.08}];
    \addplot [
      fill = espgray
    ] fill between [of = bound2 and clearance, soft clip = {domain = 1.08:2.02}];
    \addplot [
      fill = espgray
    ] fill between [of = bound3 and clearance, soft clip = {domain = 2.02:3.00}];

  \end{axis}
\end{tikzpicture}

  \caption{}%
  \label{fig:many-states}
\end{subfigure}
\begin{subfigure}{0.32\textwidth}
  \begin{tikzpicture}
  \begin{axis} [
    width = \textwidth,
    height = \textwidth,
    ymin = 1.0,
    ymax = 4.5,
    xmin = -0.2,
    xmax = 3.3,
    xtick = {0, 1.8, 3},
    xticklabels = {\( 0 \), \( \patharg[i] \), \( l \)},
    ytick = {1.9},
    yticklabels = {\( \dist[i] \)},
    axis x line = bottom,
    axis y line = left,
    axis line style = {-latex},
    xlabel = {Distance along path (\( t \))},
    ylabel = {Clearance (\( \clearance*{} \))},
    axis on top
    ]
    \node (start)  [state] at (0.0, 1.50) {};
    \node (end)    [state] at (3.0, 2.20) {};
    \node (state1) [state] at (0.6, 1.85) {};
    \node (state2) [state] at (1.2, 2.02) {};
    \node (state3) [state] at (1.8, 1.90) {};
    \node (state4) [state] at (2.4, 1.98) {};
    \draw [densely dotted] (state3) -- (1.8, 0.0);

    \addplot [
      smooth,
      tension = 0.8,
      thick,
      name path = clearance
    ] coordinates {
      (0.0, 1.5)
      (1.0, 2.0)
      (2.0, 1.9)
      (3.0, 2.2)
    } node [below = 5pt, pos = 0.3] {\( \clearance{\pathalias{\patharg}} \)};

    \addplot [
      samples = 100,
      domain = 0:3,
      name path = bound1,
      bounds
    ] {
      x + 1.5
    };
    \addplot [
      samples = 100,
      domain = 0:3,
      name path = bound2,
      bounds
    ] {
      (3 - x) + 2.2
    };
    \addplot [
      samples = 100,
      domain = 0:3,
      name path = bound3,
      bounds
    ] {
      abs(x - 0.6) + 1.85
    };
    \addplot [
      samples = 100,
      domain = 0:3,
      name path = bound4,
      bounds
    ] {
      abs(x - 1.2) + 2.02
    };
    \addplot [
      samples = 100,
      domain = 0:3,
      name path = bound5,
      bounds
    ] {
      abs(x - 1.8) + 1.90
    };
    \addplot [
      samples = 100,
      domain = 0:3,
      name path = bound6,
      bounds
    ] {
      abs(x - 2.4) + 1.98
    };

    \addplot [
      fill = espgray
    ] fill between [of = bound1 and clearance, soft clip = {domain = 0.01:0.475}];
    \addplot [
      fill = espgray
    ] fill between [of = bound3 and clearance, soft clip = {domain = 0.475:0.60}];
    \addplot [
      fill = espgray
    ] fill between [of = bound3 and clearance, soft clip = {domain = 0.60:0.985}];
    \addplot [
      fill = espgray
    ] fill between [of = bound4 and clearance, soft clip = {domain = 0.985:1.2}];
    \addplot [
      fill = espgray
    ] fill between [of = bound4 and clearance, soft clip = {domain = 1.2:1.44}];
    \addplot [
      fill = espgray
    ] fill between [of = bound5 and clearance, soft clip = {domain = 1.44:1.8}];
    \addplot [
      fill = espgray
    ] fill between [of = bound5 and clearance, soft clip = {domain = 1.8:2.14}];
    \addplot [
      fill = espgray
    ] fill between [of = bound6 and clearance, soft clip = {domain = 2.14:2.4}];
    \addplot [
      fill = espgray
    ] fill between [of = bound6 and clearance, soft clip = {domain = 2.4:2.81}];
    \addplot [
      fill = espgray
    ] fill between [of = bound2 and clearance, soft clip = {domain = 2.81:3.0}];

  \end{axis}
\end{tikzpicture}

    \caption{}%
  \label{fig:arbitrary-states}
\end{subfigure}
\caption{Illustrations of path-cost heuristics. The gray area indicates the
  error of the heuristic, which decreases as more states with known clearance
  are added.}
\end{figure}
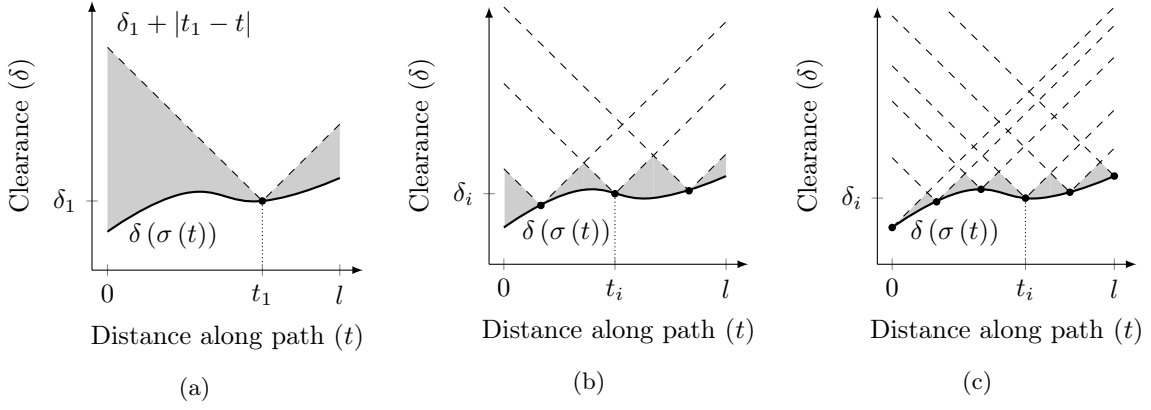


The accuracy of this heuristic improves as the number of states of known
clearance increases (Figures~\ref{fig:many-states}
and~\subref{fig:arbitrary-states}). If the start and end states are among the
states of known clearance, then the path is a chain of paths whose end states
have known clearance and Theorem~\ref{thm:arbitrary-num-states} simplifies to a
sum of Theorem~\ref{thm:end-states} over all segments
(Corollary~\ref{thm:path-cost-end-states}).

\begin{corollary}[An admissible path-cost heuristic when the clearance of the
  end states and other states of the path is known]\label{thm:path-cost-end-states}
  Let \( \path*{} \in \paths \) be a path with arc length \( \arclength \). Let
  \( 0 = \patharg[1] < \patharg[2] < \cdots < \patharg[n] = \arclength \) be a
  sequence of \( n \) numbers between \( 0 \) and \( \arclength \) whose
  associated states on the path have known clearance,
  \( \dist[i] \coloneqq \clearance{\path{\patharg[i]}} \) for
  \( i = 1, 2, 3, \hdots, n \). The reciprocal clearance cost
  \( \cost{\path*{}} \) of the path \( \path*{} \) is then lower bounded by
  \begin{align*}
     \cost{\path*{}} \geq \sum_{i = 1}^{n - 1} \ln\left( \frac{{\left( \dist[i] + \dist[i + 1] + \patharg[i + 1] - \patharg[i] \right)}^{2}}{4\dist[i]\dist[i + 1]} \right).
  \end{align*}
\end{corollary}

\begin{proof}
  (Figure~\ref{fig:arbitrary-states}) The proof follows from Theorem~\ref{thm:arbitrary-num-states} by setting
  \( \patharg[1] = 0 \) and \( \patharg[n] = \arclength \).
\end{proof}

\bibliographystyle{plainnat}
\bibliography{\string~/Dropbox/bibliography/phd}

\end{document}